\documentclass{article}
\usepackage{geometry}
\usepackage[latin9]{inputenc}
\usepackage{amsthm} 
\usepackage{amsmath} 
\usepackage{amssymb} 
\usepackage{varioref}
\usepackage[numbers]{natbib}

\usepackage{times}
\usepackage{graphicx}
\usepackage{subcaption} 

\usepackage{algorithm}
\usepackage{algorithmicx}

\usepackage{hyperref}

\makeatletter

\theoremstyle{plain}
\newtheorem{thm}{\protect\theoremname}
\theoremstyle{definition}
\newtheorem{defn}[thm]{\protect\definitionname}
\ifx\proof\undefined
\newenvironment{proof}[1][\protect\proofname]{\par
\normalfont\topsep6\p@\@plus6\p@\relax
\trivlist
\itemindent\parindent
\item[\hskip\labelsep\scshape #1]\ignorespaces
}{%
\endtrivlist\@endpefalse
}
\providecommand{\proofname}{Proof}
\fi

\theoremstyle{plain}
\newtheorem{lem}[thm]{\protect\lemmaname}
\theoremstyle{plain}
\newtheorem{cor}[thm]{\protect\corollaryname}
\theoremstyle{plain}
\newtheorem{prop}[thm]{\protect\propositionname}

\providecommand{\corollaryname}{Corollary}
\providecommand{\definitionname}{Definition}
\providecommand{\propositionname}{Proposition}
\providecommand{\theoremname}{Theorem}

\usepackage{enumitem}
\setlist{nolistsep}

\makeatother

\providecommand{\definitionname}{Definition}
\providecommand{\lemmaname}{Lemma}
\providecommand{\theoremname}{Theorem}



\usepackage{enumitem}
\setitemize[0]{leftmargin=15pt,itemindent=0pt,topsep=3pt,itemsep=1pt,partopsep=0pt,parsep=0pt}

\usepackage[justification=centering]{caption}
\usepackage{psfrag}
\usepackage{graphicx}
\usepackage{amsmath}
\usepackage{amssymb}
\usepackage{amsthm}
\usepackage{amsbsy}
\usepackage{algpseudocode}
\usepackage{longtable}
\usepackage{verbatim}
\usepackage{icomma} %intelligent comma in math
\usepackage{ifthen}
\usepackage{booktabs}

\graphicspath{{.},{figs/}}

\providecommand{\newdef}[2]{\newtheorem{#1}{#2}}

\newdef{definition}{Definition} 
\newdef{observation}{Observation} %[section]
\newdef{example}{Example} %[section]
\providecommand{\eq}[1]{\eqref{eq:#1}}

\providecommand{\alg}[1]{Algorithm~\ref{alg:#1}}

\providecommand{\funcName}[1]{\textsc{#1}}

\begin{document} 
\title{Learning in POMDPs with Monte Carlo Tree Search\footnote{
    This is an extended version of a paper that was published at ICML'2017.}
}

\author{
    Sammie Katt\\
    Northeastern University
\and 
    Frans A. Oliehoek\\ 
    University of Liverpool
\and 
    Christopher Amato\\
    Northeastern University
}

\date{}
\maketitle

\begin{abstract} 
    The POMDP is a powerful framework for reasoning under outcome and information
uncertainty, but constructing an accurate POMDP model is difficult.
Bayes-Adaptive Partially Observable Markov Decision Processes (BA-POMDPs)
extend POMDPs to allow the model to be learned during execution. BA-POMDPs are
a Bayesian RL approach that, in principle, allows for an optimal trade-off
between exploitation and exploration. Unfortunately, BA-POMDPs are currently
impractical to solve for any non-trivial domain. In this paper, we extend the
Monte-Carlo Tree Search method POMCP to BA-POMDPs and show that the resulting
method, which we call BA-POMCP, is able to tackle problems that previous
solution methods have been unable to solve. Additionally, we introduce several
techniques that exploit the BA-POMDP structure to improve the efficiency of
BA-POMCP along with proof of their convergence.

\end{abstract} 

\section{Introduction}\label{sec:introduction}
The Partially Observable Markov Decision Processes (POMDP)~\cite{POMDP} is a
general model for sequential decision making in stochastic and partially
observable environments, which are ubiquitous in real-world problems.  A key
shortcoming of POMDP methods is the assumption that the dynamics of the
environment are known a priori. In real-world applications, however, it may be
impossible to obtain a complete and accurate description of the system.
Instead, we may have uncertain prior knowledge about the model. When lacking a
model, a prior can be incorporated into the POMDP problem in a principled way,
as demonstrated by the \textit{Bayes-Adaptive POMDP} framework~\cite{BAPOMDP}. 

The BA-POMDP framework provides a Bayesian approach to decision making by
maintaining a probability distribution over possible models as the agent acts
in an online reinforcement learning setting~\cite{duff2002optimal,
wyatt2001exploration}. This method casts the Bayesian reinforcement learning
problem into a POMDP planning problem where the hidden model of the environment
is part of the state space.  Unfortunately, this planning problem becomes very
large, with a countably infinite state space over all possible models, and as
such, current solution methods are not scalable or perform
poorly~\cite{BAPOMDP}. 

Online and sample-based planning has shown promising performance on non-trivial
POMDP problems~\cite{ross2008online}. Online methods reduce the complexity by
considering the relevant (i.e., reachable) states only, and sample-based
approaches tackle the complexity issues through approximations in the form of
simulated interactions with the environment.  Here we modify one of those
methods, \textit{Partial Observable Monte-Carlo Planning} (POMCP) \cite{POMCP}
and extend it to the \textit{Bayes-Adaptive} case, leading to a novel approach:
BA-POMCP. 

In particular, we improve the sampling approach by exploiting the structure of
the BA-POMDP resulting in \emph{root sampling} and \emph{expected models}
methods. We also present an approach for more efficient model representation,
which we call \emph{linking states}. Lastly, we prove the correctness of our
improvements, showing that they converge to the true BA-POMDP solution. As a
result, we present methods that significantly improve the scalability of
learning in BA-POMDPs, making them practical for larger problems.

\section{Background}\label{sec:background}
First, we discuss POMDPs and BA-POMDPs in respectively
Section~\ref{ssec:frameworks} and \ref{ssec:bapomdp}.

\subsection{POMDPs}\label{ssec:frameworks}
Formally, a POMDP is described by a tuple ($S$, $A$, $Z$, $D$, $R$, $\gamma$,
$h$), where
        $S$ is the set of states of the environment; $A$ is the set of actions;
        $Z$ is the set of observations; $D$ is the `dynamics function' that
        describes the dynamics of the system in the form of transition
        probabilities $D(s',z|s,a)$;\footnote
        {
            This formulation allows for easier notation and generalizes the
            typical formulation with separate transition $T$ and observation
            functions $O$: $D = \langle T, O \rangle$. In our experiments, we
            do employ this typical factorization.
        }
        $R$ is the immediate reward function $R(s, a)$ that describes the
        reward of selecting $a$ in $s$; $\gamma \in [0,1)$ is the discount
        factor; and $h$ is the horizon of an episode in the system. 

The goal of the agent in a POMDP is to maximize the expected cumulative
(discounted) reward, also called the expected return. The agent has no direct
access to the system's state, so it can only rely on the
\emph{action-observation history} $h_t=\langle a_0,z_1,\ldots,a_{t-1},z_t
\rangle$ up to the current step $t$. It can use this history to maintain a
probability distribution over the state, also called a belief, $b(s)$.  A
solution to a POMDP is then a mapping from a belief $b$ to an action $a$, which
is called a $policy$ $\pi$. Solution methods aim to find an optimal policy, a
mapping from a belief to an action with the highest possible expected return. 

The agent maintains its belief during execution through \textit{belief
updates}. A belief update calculates the posterior probability of the state
$s'$ given the previous belief over $s$ and action-observation pair $\langle
a,z \rangle$: $b'(s) = P(s' | b(s), a, z)$. This operation is infeasible for
large spaces because it enumerates over the entire state space. A common
approximation method is to represent the belief with a (unweighted)
\textit{particle filter}~\cite{Thrun99}. A particle filter is a collection of
$K$ particles (states). Each particle represents a probability of
$\frac{1}{K}$; if a specific state $x$ occurs $n$ times in a particle filter,
then $P(x) = \frac{n}{K}$. The precision of the filter is determined by the
number of particles $K$.  To update such a belief after execution of action $a$
and observation $z$, a standard approach is to  utilize \textit{rejection
sampling}: the agent repeatedly samples a state $s$ from its belief, then
simulates the execution of $a$ on $s$ through $D$, and receives a (simulated)
new state $s_{sim}'$ and observation $z_{sim}$. $s'$ is added to the new belief
only when $z_{sim}$ equals $z$, and rejected otherwise. This process repeats
until the new belief contains $K$ particles.

\textit{Partially Observable Monte-Carlo Planning} (POMCP)~\cite{POMCP}, is a
scalable method which extends Monte Carlo tree search (MCTS) to solve POMDPs.
POMCP is one of the leading algorithms for solving general POMDPs. At each time
step, the algorithm performs online planning by incrementally building a
lookahead tree that contains $Q(h,a)$, where $h$ is the action-observation
history-path to reach that node.  It samples hidden states $s$ at the root node
(called `root sampling') and uses that state to sample a trajectory that first
traverses the lookahead tree and then performs a (random) rollout. The return
of this trajectory is used to update the statistics for all visited nodes.
These statistics include the number of times an action has been taken at a
history ($N(h,a)$) and estimated value of being in that node ($Q(h,a)$), based
on an average over the returns.

Because this lookahead tree can be very large, the search is directed to the
relevant parts by selecting the actions inside the tree that maximize the
`upper confidence bounds'~\cite{auer2002finite}: $ U(h,a) = Q(h, a) + c
\sqrt{\log(N(h) + 1) / N(h,a)}.$ Here, $N(h)$ is the number of times the
history has been visited. At the end of each simulation, the discounted
accumulated return is used to update the estimated value of all the nodes in
the tree that have been visited during that simulation.  POMCP terminates after
some criteria has been met, typically defined by a maximum number of
simulations or allocated time. The agent then picks the action with the highest
estimated value $(\max_a Q(b,a))$. POMCP can be shown to converge to an
$\epsilon$-optimal value function. Moreover, the method has demonstrated good
performance in large domains with a limited number of simulations. The
extension of POMCP that is used in this work is discussed in
Section~\ref{sec:bapomcp}.

\subsection{BA-POMDPs}\label{ssec:bapomdp}
Most research concerning POMDPs has considered the task of \emph{planning}:
given a full specification of the model, determine an optimal policy (e.g.,
\cite{POMDP,Shani2012}). However, in many real-world applications, the model is
not (perfectly) known in advance, which means that the agent has to learn about
its environment during execution. This is the task considered in
\emph{reinforcement learning} (RL)~\cite{SuttonBarto98}.

A fundamental RL problem is the difficulty of deciding whether to select
actions in order to learn a better model of the environment, or to exploit the
current knowledge about the rewards and effects of actions. In recent years,
Bayesian RL methods have become popular because they can provide a principled
solution to this exploration/exploitation
trade-off~\cite{wyatt2001exploration,duff2002optimal,engel2005reinforcement,poupart2008model,vlassis2012bayesian}.

In particular, we consider the framework of Bayes-Adaptive POMDPs
\cite{ross2007bayes,BAPOMDP}. BA-POMDPs use Dirichlet distributions to model
uncertainty over transitions and observations\footnote
    {
        \cite{ross2007bayes,BAPOMDP} follow the standard $T$ \& $O$ POMDP representations,
        but we use our combined $D$ formalism.
    }
(typically assuming the reward function is chosen by the designer and is
known). In particular, if the agent could observe both states and observations,
it could maintain a vector $\chi$ with the counts of the occurrences for all
$\langle s, a, s', z \rangle $ tuples. We write $\chi_{sa}^{s'z}$ for the
number of times that $\langle s,a \rangle$ is followed by $\langle s',z
\rangle$.

While the agent cannot observe the states and has uncertainty about the actual
count vector, this uncertainty can be represented using regular POMDP
formalisms. That is, the count vector is included as part of the hidden state
of a specific POMDP, called a BA-POMDP. Formally, a BA-POMDP is a tuple $
\langle \bar{S}, A, \bar{D}, \bar{R}, Z, \gamma, h \rangle $ with some modified
components in comparison to the POMDP.  While the observation and action space
remain unchanged, the state (space) of the BA-POMDP now includes Dirichlet
parameters: $\bar{s} = \langle s, \chi \rangle$, which we will refer to as
augmented states. The reward model remains the same, since it is assumed to be
known, $\bar{R}(\langle s',\chi'),a)=R(s,a)$. The dynamics functions,
$\bar{D}$, however, is described in terms of the counts in $\bar{s}$, and is
defined as follows 
\begin{equation}
D_\chi( s',z | s, a) \triangleq \mathbf{E} [ D( s',z | s, a) | \chi ] = 
    \frac{
            \chi_{sa}^{s'z}
        }{
        \sum_{s'z} \chi_{sa}^{s'z}
    }.
    \label{eq:DChi}
\end{equation}    
These expectations can now be used to define the transitions for the BA-POMDP.
If we let $\delta_{sa}^{s'z}$ denote a vector of the length of $\chi$
containing all zeros except for the position corresponding to $\langle s,a,s',z
\rangle$ (where it has a one), and if we let $\mathbb{I}_{a,b}$ denote the
Kronecker delta that indicates (is 1 when) $a=b$, then we can define $\bar{D}$
as $ \bar{D}( s',\chi',z | s,\chi, a) = D_\chi(s',z | s, a)
\mathbb{I}_{\chi',\chi+\delta_{sa}^{s'z}}.$ 

Remember that these counts are not observed by the agent, since that would
require observations of the state. The agent can only maintain belief over
these count vectors. Still, when interacting with the environment, \emph{the
ratio of the true---but unknown---count vectors will converge to coincide with
the true transition and observation probabilities in expectation}. It is
important to realize, however, that this convergence of count vector ratios
does not directly imply learnability by the agent: even though the ratio of the
count vectors of the true hidden state will converge, \emph{the agent's belief
over count vectors might not}.

BA-POMDPs are infinite state POMDP models and thus extremely difficult to
solve. \citet{BAPOMDP} introduced a technique to convert such models to finite
models, but these are still very large. Therefore, Ross et al.~propose a simple
lookahead planner to solve BA-POMDPs in an online manner. This approach
approximates the expected values associated with each action at the belief by
applying a lookahead search of depth $d$. This method will function as the
comparison baseline in our experiments, as no other BA-POMDP solution methods
have been proposed.

\section{BA-POMDPs via Sample-based Planning}\label{sec:bapomcp}
Powerful methods, such as POMCP \cite{POMCP}, have significantly improved the
scalability of POMDP solution methods.  At the same time the most practical
solution method for BA-POMDPs, the aforementioned lookahead algorithm, is quite
limited in dealing with larger problems. POMDP methods have rarely been applied
to BA-POMDPs \cite{Amato15AAAI}, and no systematic investigation of their
performance has been conducted. In this paper, we aim to address this void, by
extending POMCP to BA-POMDPs, in an algorithm that we refer to as
\emph{BA-POMCP}.  Moreover, we propose a number of novel adaptations to
BA-POMCP that exploit the structure of the BA-POMDP. In this section, we first
lay out the basic adaptation of POMCP to BA-POMDPs and then describe the
proposed modifications that improve its efficiency. 

\paragraph{BA-POMCP}
BA-POMCP, just like POMCP, constructs a lookahead tree through simulated
experiences (\alg{bapomcpOuterloop}).
\providecommand{\commentSymb}{//}
\begin{algorithm}[t]
    \caption{
        \funcName{BA-POMCP}$(\bar{b},num\_sims)$
    }\label{alg:bapomcpOuterloop}
    \begin{algorithmic}[1]
\small                
        \State{\commentSymb{}$\bar{b}$ is an augmented belief (e.g., particle filter)}
        \State{$h_0 \gets ()$}  \Comment{The empty history (i.e., now)}
        \For{$i \gets 1\dots{}num\_sims$}
            \State{\commentSymb{}First, we \emph{root sample} an (augmented) state:}
            \State{$ \bar{s} \gets \funcName{Sample}(\bar{b}) $} \Comment{reference to a particle}
            \State{\label{stepCopy}$ \bar{s}' \gets \funcName{Copy}(\bar{s})  $}
            \State{$ \funcName{Simulate}(\bar{s}', 0, h_0)    $}
        \EndFor
        \State{$a \gets \funcName{GreedyActionSelection}(h_0)$}
        \State{\textbf{return} $a$}
    \end{algorithmic}
\end{algorithm} 
In BA-POMDPs, however, the dynamics of the system are inaccessible during
simulations, and the belief is a probability distribution over augmented
states. BA-POMCP, as a result, must sample augmented states from the belief
$\bar{b}$, and use copies of those states ($\bar{s}=\langle s,\chi \rangle$)
for each simulation (\alg{bapomcp_simulate}). We will refer to this as
\emph{root sampling of the state} (line~\ref{stepCopy}). The copy is necessary,
as otherwise the \funcName{Step} function in \alg{bapomcp_simulate} would alter
the belief $\bar{b}$. It is also expensive, for $\chi$ grows with the state,
action and observation space, to $|S|^2 \times |A| \times |\Omega|$ parameters.
    \footnote{It is $|S|^2 \times |A| + |S| \times |A| \times |\Omega|$ when
    assuming $D$ is factored in $T$ \& $O$} 
    In practice, this operation becomes a bottleneck to the runtime of BA-POMCP
    in larger domains.

\algrenewcommand{\algorithmiccomment}[1]{\hfill\commentSymb#1}
\begin{algorithm}[t]
    \caption{
        \funcName{Simulate}$(\bar{s}, d, h)$
    }\label{alg:bapomcp_simulate}
    \begin{algorithmic}[1]
\small                
        \If{$\funcName{IsTerminal}(h)$ $||$ $d == max\_depth$}
                \State{return $0$}
            \EndIf{}
        \State{\commentSymb{}Action selection uses statistics stored at node $h$:} 
        \State{$a \gets \funcName{UCBActionSelection}(h)$} 
        \State{$R \gets R(\bar{s},a) $ }
        \State{$z \gets \funcName{Step}(\bar{s}, a)$} \Comment{\emph{modifies} $\bar{s}$ to sampled next state}
        \State{$h' \gets (h,a,z)$}
        \If{$h'  \in Tree$} 
            \State{$r \gets R + \gamma \; \funcName{Simulate}(\bar{s}, d+1, h')$} 
        \Else{}
            \State{$\funcName{ConstructNode}(h')$} \Comment{Initializes statistics}
            \State{$r \gets R + \gamma \; \funcName{Rollout}(\bar{s}, d+1, h')$} 
        \EndIf{}
        \State{\commentSymb{}Update statistics:}
        \State{$N(h,a) \gets N(h,a) + 1$} 
        \State{$Q(h,a) \gets \frac{N(h,a)-1}{N(h,a)} Q(h,a) +
        \frac{1}{N(h,a)} r$}
        \State{\textbf{return} $r$}
    \end{algorithmic}
\end{algorithm} 

To apply POMCP on BA-POMDPs, where the dynamics are unknown, we modify the
\funcName{Step} function, proposing several variants. The most straightforward
one, \funcName{BA-POMCP-Step} is employed in what we refer to as `BA-POMCP'. This
method, shown in~\alg{naive}, is similar to BA-MCP \cite{guez2012efficient}:
essentially, it samples a dynamic model $D_{sa}$ which specifies probabilities
$\Pr(s',z|s,a)$ and subsequently samples an actual next state and observation
from that distribution. Note that the underlying states and observations are
all represented simply as an index, and hence the assignment on line
\ref{naive:assignment_of_state} is not problematic. However, the cost of the
model sampling operation in line~\ref{line:model_sample} is.

\begin{algorithm}[t]
    \caption{
        \funcName{BA-POMCP-Step}$(\bar{s}=\langle s, \chi \rangle, a)$
    }\label{alg:naive}
    \begin{algorithmic}[1]
\small                
        \State{\label{line:model_sample}$D_{sa} \sim \chi_{sa}$}
        \State{$\langle s',z \rangle \sim D_{sa}$} 
        \State{\commentSymb{}In place updating of $\bar{s}=\langle s, \chi \rangle$}
        \State{\label{alg:naive:count_update}$\chi_{sa}^{s'z} \gets \chi_{sa}^{s'z}+1 $ } 
        \State{\label{naive:assignment_of_state}$s \gets s' $}  
        \State{\textbf{return} $z$}
    \end{algorithmic}
\end{algorithm} 

\paragraph{Root Sampling of the Model} BA-MCP~\cite{guez2012efficient}
addresses the fully observable BRL problem by using POMCP on an augmented state
$\bar{s}=\langle s, T \rangle$, consisting of the observable state, as well as
the hidden true transition function $T$.  Application of POMCP's root sampling
of state in this case leads to `root sampling of a transition function': Since
the true transition model $T$ does not change during the simulation, one is
sampled at the root and used during the entire simulation. In the BA-POMCP
case, root sampling of a state $\bar{s}=\langle s, \chi \rangle$ does not lead
to a same interpretation: no model, but counts are root sampled and they
\emph{do} change over time.

We use this as inspiration to introduce a similar, but clearly different,
(since this is not \emph{root sampling of state}) technique called \emph{root
sampling of the model} (which we will refer to as just `root sampling'). The
idea is simple: every time we root sample a state $\bar{s}=\langle s, \chi
\rangle \sim \bar{b}$ at the beginning of a simulation (line $5$ in~
\alg{bapomcpOuterloop}), we directly sample a $\dot{D} \sim Dir(\chi) $, which
we will refer to as the root-sampled model $\dot{D}$ and it is used for the
rest of the simulation. 

We denote this root sampling in BA-POMCP as `R-BA-POMCP'. The approach is
formalized by \funcName{R-BA-POMCP-Step} (\alg{rs}). Note that no count updates
take place (cf.~line~\ref{alg:naive:count_update} in~\alg{naive}). This
highlights an important advantage of this technique: since the counts are not
used in the remainder of the simulation, the copy of counts (as part of line
\ref{stepCopy} of~\alg{bapomcpOuterloop}) can be avoided altogether. Since this
copy operation is costly, especially in larger domains, where the number of
states, action and observations and the number of counts is large, this can
lead to significant savings. Finally, we point out that, similar to what
\citet{guez2012efficient} propose, $\dot{D}$ can be constructed lazily: the
part of the model $\dot{D}$ is only sampled when it becomes necessary. 

\begin{algorithm}[t]
    \caption{
        \funcName{R-BA-POMCP-Step
        }$(\bar{s} = \langle s, \chi \rangle, a)$
    }\label{alg:rs}
    \begin{algorithmic}[1]
\small                
        \State{\commentSymb{}Sample from the root sampled model}
        \State{$s',z \sim \dot{D}_{s,a}$}
        \State{\label{RS:assignment_of_state}$s \gets s' $}  
        \State{\textbf{return} $z$}
    \end{algorithmic}
\end{algorithm} 

The transition probabilities during R-BA-POMCP differ from those in BA-POMCP,
and it is not obvious that a policy based on R-BA-POMCP maintains the same
guarantees. We prove in Section~\ref{sec:theory} that the solution of
R-BA-POMCP in the limit converges to that of BA-POMCP.

\paragraph{Expected models during simulations}
The second, complementary, adaptation modifies the way models are sampled from
the root-sampled counts in \funcName{Step}. This version samples the
transitions from the \textit{expected} dynamics $D_\chi$ given in~\eq{DChi},
rather than from a sampled dynamics function $D \sim Dir(\chi)$.  The latter
operation is relatively costly, while constructing $D_\chi$ is very cheap.  In
fact, this operation is so cheap, that it is more efficient to (re-)calculate
it on the fly rather than to actually store $D_\chi$. This approach is shown
in~\alg{exm}.

\begin{algorithm}[t]
    \caption{
        \funcName{E-BA-POMCP-Step}$(\bar{s} = \langle s, \chi \rangle, a)$
    } \label{alg:exm}
    \begin{algorithmic}[1]
        \small
        \State{\commentSymb{}Sample from the \textit{Expected} model}
        \State{$s',z \sim D_\chi(\cdot, \cdot | s,a) $}
        \State{$ \chi_{sa}^{s'z} \gets \chi_{sa}^{s'z}+1 $ } 
        \State{$s \gets s'$}
        \State{\textbf{return} $z$}
    \end{algorithmic}
\end{algorithm} 

\paragraph{Linking States}\label{sssec:linking_states} Lastly, we propose a
specialized data structure to encode the augmented BA-POMDP states. The
structure aims to optimize for the complexity of the count-copy operation in
line \ref{stepCopy} of~\alg{bapomcpOuterloop} while allowing modifications to
$\bar{s}$. The linking state $s_l$ is a tuple of a system state, a pointer (or
\emph{link}) to an unmodifiable set of counts $\chi$ and a set of
\textit{updated} counts $\langle s, l, \delta \rangle$.  $l$ is a pointer to
some set of counts $\chi$, which remain unchanged during count updates (such as
in the \funcName{Step} function), and instead are stored in the set of updated
counts, $\delta$, as shown in~\alg{ls}. The consequence is that the linking
state copy-operation can safely perform a shallow copy of the counts $\chi$,
and must only consider $\delta$, which is assumed to be much smaller.

Linking states can be used during the (rejection-sample-based) belief update at
the beginning of each real time step. While the root-sampled augmented states
(including $\delta$ in linking states) are typically deleted at the end of each
simulation during L-BA-POMCP, each belief update potentially increases the size
of $\delta$ of each particle.  Theoretically, the number of updated counts
represented in $\delta$ increases and the size of $\delta$ may (eventually)
grow similar to the size of $\chi$.  Therefore, at some point, it is necessary
to construct a new $\chi'$ that combines $\chi$ and $\delta$ (after which
$\delta$ can be safely emptied). We define a new parameter for the maximum size
of $\delta$, $\lambda$, and condition to merge only if the size of $\delta$
exceeds $\lambda$. We noticed that, in practice, the number of merges is much
smaller than the amount of copies in BA-POMCP. We also observed in our
experiments that it is often the case that a specific (small) set of
transitions are notably more popular than others and that $\delta$ grows quite
slowly.
\begin{algorithm}[t]
    \caption{
        \funcName{L-BA-POMCP-Step}$(s_l = \langle s, l, \delta \rangle, a)$
    }\label{alg:ls}
    \begin{algorithmic}[1]
\small                
        \State{\label{alg:ls:sample}$D \sim \langle l, \delta \rangle$}
        \State{$s',z \sim D_{s,a}$}
        \State{$s \gets s'$}
        \State{$\delta_{sa}^{s'z} \gets \delta_{sa}^{s'z}+1$} 
        \State{\textbf{return} $z$}
    \end{algorithmic}
\end{algorithm}

\section{Theoretical Analysis}\label{sec:theory}

\global\long\def\cip{\overset{p}{\rightarrow}}

\global\long\def\fh{H}

\global\long\def\fhD#1{\fh_{#1}}

\global\long\def\fhDP#1#2{\fh_{#1}^{(#2)}}

\global\long\def\aoh{h}

\global\long\def\aohD#1{\aoh_{#1}}

\global\long\def\sh{STATEHISTORY}

\global\long\def\shD#1{s_{0:#1}}

\global\long\def\KD#1#2{{\mathbb{I}_{{#1},{#2}}}}

\global\long\def\DD#1#2{{\mathbb{I}_{{#1},{#2}}}}

\global\long\def\sbr{\bar{s}}

\global\long\def\sba{\check{s}}

\providecommand{\BAPOMCP}{BA-POMCP} \providecommand{\RSBAPOMCP}{RS-BA-POMCP}
Here, we analyze the proposed root sampling of the dynamics
function and expected transition techniques, and demonstrate they
converge to the solution of the BA-POMDP. These main steps of this
proof are similar to those in \cite{POMCP}. We point out however, that
the technicalities of proving the components are far more involved.

The convergence guarantees of the original POMCP method are based
on showing that, for an arbitrary rollout policy $\pi$, the \emph{POMDP
rollout distribution }(the distribution over full histories when performing
root sampling of state) is equal to the \emph{derived MDP rollout
distribution (}the distribution over full histories when sampling
in the belief MDP). Given that these are identical it is easy to see
that the statistics maintained in the search tree will converge to
the same number in expectation. As such, we will show a similar result
here for expected transitions (`expected' for short) and root sampling
of the dynamics function (`root sampling' below). 

We define $\fhD 0$ as the \emph{full} history (also including states) at the
root of simulation, $\fhD d$ as the full history of a node at depth $d$ in the
simulation tree, and $\chi(\fhD d)$ as the counts induced by $\fhD d$. We then
define the rollout distributions:
\begin{defn}\label{theory-def:The-full-history-BA-POMDP}
    The \emph{expected} \emph{full-history }\textbf{\emph{expected transition
    }}\emph{BA-POMDP rollout distribution} is the distribution over full
    histories of a BA-POMDP, when performing Monte-Carlo simulations according
    to a policy $\pi$. It is given by
    \begin{equation}
        P^{\pi}(\fhD{d+1})=D_{\chi(\fhD d)}(s_{d+1},z_{d+1}|a_{s},s_{d})\pi(a_{d}|\aohD d)P^{\pi}(\fhD d)\label{theory-eq:fh-BA-POMDP_rollout-distr}
    \end{equation}
    with $P^{\pi}(\fhD 0)=b_{0}(\left\langle s_{0},\chi_{0}\right\rangle )$ the
    belief `now' (at the root of the online planning).
\end{defn}
Note that there are two expectations in the above definition: `expected
transitions' mean that transitions for a history $\fhD d$ are \emph{sampled
}from $D_{\chi(\fhD d)}$. The other `expected' is the expectation of those
samples (and it is easy to see that this will converge to the expected
transition probabilities $D_{\chi(\fhD d)}(s_{d+1},z_{d+1}|a_{s},s_{d})$).  For
root sampling of the dynamics model, this is less straightforward, and we give
the definition in terms of the empirical distribution:
\begin{defn} \label{theory-def:The-full-history-RS-BA-POMDP}
    The \emph{empirical full-history }\textbf{\emph{root-sampling (RS)}}\emph{
        BA-POMDP rollout distribution} is the distribution over full histories
    of a BA-POMDP, when performing Monte-Carlo simulations according to a
    policy $\pi$ \textbf{in combination with root sampling} of the dynamics
    model $D$. This distribution, for a particular stage $d$, is given by \[
        \tilde{P}_{K}^{\pi}(\fhD
        d)\triangleq\frac{1}{K_{d}}\sum_{i=1}^{K_{d}}\mathbb{I}_{\left\{ \fhD
        d=\fhDP di\right\} }, \] where
\end{defn}
\begin{itemize}
    \item $K$ is the number of simulations that comprise the empirical
        distribution,
    \item $K_{d}$ is the number of simulations that reach depth $d$ (not all
        simulations might be equally long),
    \item $\fhDP di$ is the history specified by the $i$-th particle at stage
        $d$.
\end{itemize}
Now, our main theoretical result is that these distributions are the
same in the limit of the number of simulations:
\begin{thm}
\label{lem:Ph_converges-in-prob}The\emph{ }full-history RS-BA-POMDP
rollout distribution (Def.~\ref{theory-def:The-full-history-RS-BA-POMDP})
converges in probability to the quantity of Def.~\ref{theory-def:The-full-history-BA-POMDP}:
\begin{equation}
\forall_{\fhD d}\quad\tilde{P}_{K_{d}}^{\pi}(\fhD d)\cip P^{\pi}(\fhD d).\label{theory-eq:lemma_RS-rollout_converges_to_rollout}
\end{equation}
\end{thm}
\begin{proof}
    The proof is listed in appendix A.
\end{proof}
\begin{cor}
\label{thm:convergence_identical_solution} Given suitably chosen
exploration constant (e.g., $c>\frac{Rmax}{1-\gamma}$), BA-POMCP
with root-sampling of dynamics function converges in probability to
the expected transition solution.
\begin{proof}
Since Theorem \ref{lem:Ph_converges-in-prob} guarantees the distributions
over histories are the same in the limit, they will converge to the
same values maintained in the tree.
\end{proof}
\end{cor}
Finally, we see that these are solutions for the BA-POMDP:
\begin{cor}
\label{thm:convergence_BAPOMDP_solution} BA-POMCP with expected
transitions sampling, as well as with root sampling of dynamics function
converge to an $\epsilon$-optimal value function of a BA-POMDP: $V(\left\langle s,\chi\right\rangle ,\aoh)\cip V_{\epsilon}^{*}(\left\langle s,\chi\right\rangle ,\aoh)$,
where $\epsilon=\frac{precision}{1-\gamma}$. \end{cor}
\vspace{-2mm}
\begin{proof}
A BA-POMDP is a POMDP, so the analysis from \citet{POMCP}
applies to the BA-POMDP, which means that the stated guarantees hold
for BA-POMCP. The BA-POMDP is stated in terms of expected transitions,
so the theoretical guarantees extend to the expected transition BA-POMCP,
which in turn via corollary~\ref{thm:convergence_identical_solution} implies
that the theoretical guarantees extend to \RSBAPOMCP.
\end{proof}
\vspace{-2mm}
Finally, we note that linking states does not affect they way that
sampling is performed at all:
\begin{prop}
Linking states does not affect convergence of BA-POMCP.
\end{prop}

\section{Empirical Evaluation}\label{sec:evaluation}
\paragraph{Experimental setup} In this section, we evaluate our algorithms on a
small toy problem, the well-known Tiger problem~\cite{cassandra1994acting} and
test scalability on a larger domain: the Partially Observable Sysadmin
(POSysadmin) problem. In POSysadmin, the agent acts as a system administrator
with the task of maintaining a network of $n$ computers.  Computers are either
`working' or `failing', which can be deterministically resolved by `rebooting'
the computer.  The agent does not know the state of any computer, but can
`ping' any individual computer. At each step, any of the computers can `fail'
with some probability $f$. This leads to a state space of size $2^n$, an action
space of $2n + 1$, where the agent can `ping' or `reboot' any of the computers,
or `do nothing', and an observation space of $3$ ($\{NULL, failing,
working\}$). The `ping' action has a cost of $1$ associated with it, while
rebooting a computer costs $20$ and switches the computer to `working'.
Lastly, each `failing' computer has a cost of $10$ at each time step.

We conducted an empirical evaluation with aimed for $3$ goals: The first goal
attempts to support the claims made in Section~\ref{sec:theory} and show that
the adaptations to BA-POMCP do not decrease the quality of the resulting
policies. Second, we investigate the runtime of those modifications to
demonstrate their contribution to the efficiency of BA-POMCP. The last part
contains experiments that directly compare the performance per action selection
time with the baseline approach of~\citet{BAPOMDP}.  For brevity,
Table~\ref{tab:default_parameters} describes the default parameters for the
following experiments. It will be explicitly mentioned whenever different
values are used. 

\begin{table}[t]
    \caption{Default experiment parameters}\label{tab:default_parameters}
    \centering
    \begin{tabular}{l c}
        \small
        Parameter & Value \\ 
        \toprule
        $\gamma$ & $0.95$ \\
        $horizon$ (h) & $20$ \\
        \# particles in belief & $1000$ \\
        $f$ (computer fail \%) & $0.1$ \\
        exploration const & $h * (\max(R) - \min(R))$ \\
        \# episodes & $100$ \\
        $\lambda =$ \# updated counts & $30$ \\
        \bottomrule
    \end{tabular}
\end{table}

\paragraph{BA-POMCP variants}
Section~\ref{sec:theory} proves that the solutions of the proposed
modifications (root-sampling (R-), expected models (E-) and linking states
(L-)) in the limit converge to the solution of BA-POMCP. Here, we investigate
the behaviour of these methods in practice. For the Tiger problem, the agent's
initial belief over the transition model is correct (i.e., counts that
correspond to the true probabilities with high confidence), but it provides an
uncertain belief that underestimates the reliability of the observations.
Specifically, it assigns $5$ counts to hearing the correct observation and $3$
counts to incorrect: the agent initially beliefs it will hear correctly with a
probability of $62.5\%$. The experiment is run for with $100$, $1000$ \& $1000$
simulations and all combinations of BA-POMCP adaptations. 

\begin{figure}[t!]
    \centering
    \begin{subfigure}{0.5\textwidth}
        \centering
        \includegraphics[width=1\textwidth]{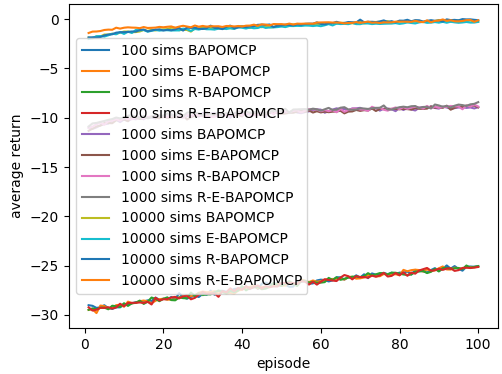}
        \caption{The average discounted return of BA-POMCP over $100$ episodes on the Tiger
        problem for $100$, $1000$ \& $10000$ simulations}
        \label{fig:tiger_equivalence}
    \end{subfigure}%
    ~
    \begin{subfigure}{0.5\textwidth}
        \centering
        \includegraphics[width=1\textwidth]{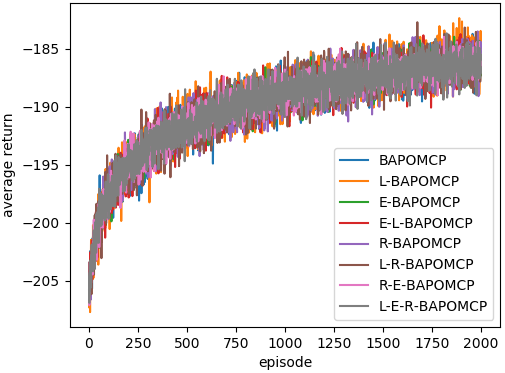}
        \caption{The average discounted of $100$ simulations of BA-POMCP per time
        episodes on the Sysadmin problem} \label{fig:sysadmin_equivalence}
    \end{subfigure}
\end{figure}

Figure~\ref{fig:tiger_equivalence} plots the average return over $10000$ runs
for a learning period of $100$ episodes for Tiger. The key observation here is
two-fold.  First, all methods improve over time through refining their
knowledge about $D$. Second, there are three distinct clusters of lines, each
grouped by the number of simulations. This shows that all $3$ variants
(R/L/E-BA-POMCP) lead to the same results.

We repeat this investigation with the ($3$-computer) POSysadmin problems, where
we allow $100$ simulations per time step. In this configuration, the network
was fully connected with a failure probability $f = 0.1$.  The (deterministic)
observation function is assumed known a priori, but the prior over the
transition function is noisy as follows: for each count $c$, we take the true
probability of that transition (called $p$) and (randomly) either subtract or
add $.15$. Note that we do not allow transitions with non-zero probability to
fall below $0$ by setting those counts to $0.001$. Each Dirichlet distribution
is then normalized the counts to sum to $20$. With $3$ computers, this results
in $|S| \times |A| = 8 \times 7 = 56$ noisy Dirichlet distributions of $|S| =
8$ parameters.

Figure~\ref{fig:sysadmin_equivalence} shows how each method is able to increase
its performance over time for POSysadmin. Again, the proposed modifications do
not seem to alter the solution quality for a specific number of simulations.

\paragraph{BA-POMCP scalability} 
While the previous experiments indicate that the three adaptations produce
equally good policies, they do not support any of the efficiency claims made in
Section~\ref{sec:bapomcp}. Here, we compare the scalability of BA-POMCP on the
POSysadmin problem. The proposed BA-POMCP variants are repeatedly run for $100$
episodes on instances of POSysadmin of increasing network size ($3$ to $10$
computers), and we measure the average action selection time required for
$1000$ simulations. Note that the experiments are capped to allow up to $5$
seconds per action selection, demonstrating the problem size that a specific
method can perform $1000$ simulations in under $5$ seconds.

\begin{figure}[t]
    \centering
        \includegraphics[width=0.475\textwidth]{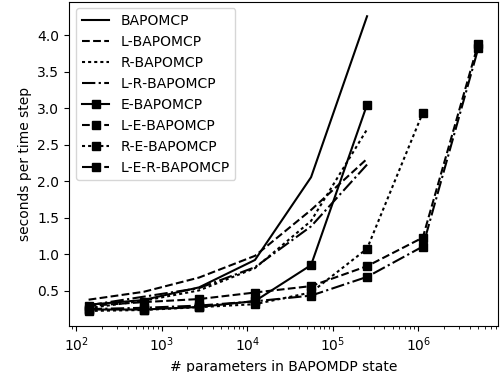}
        \caption{The average amount of seconds required for BA-POMCP given the 
        the log of the amount of parameters (size) in the POSysadmin problem}
        \label{fig:sysadmin_time_vs_parameters}
\end{figure}

Figure~\ref{fig:sysadmin_time_vs_parameters} shows that BA-POMCP takes less
than $0.5$ seconds to perform $1000$ simulations on an augmented state with
approximately $150$ parameters ($3$ computers), but is quickly unable to solve
larger problems, as it requires more than $4$ seconds to plan for a BA-POMDP with
$200000$ counts. BA-POMCP versions with a single adaptation are able to solve
the same problems twice as fast, while combinations are able to solve much
larger problems with up to $5$ million parameters ($10$ computers). This
implies not only that each individual adaptation is able to speed up BA-POMCP,
but also that they complement one another.

\paragraph{Performance}
The previous experiments first show that the adaptations do not decrease the
policy quality of BA-POMCP and second that the modified BA-POMCP methods
improve scalability. Here we put those thoughts together and directly consider
the performance relative to the action selection time. In these experiments we
take the average return over multiple repeats of $100$ episodes and plot them
according to the time required to reach such performance. Here BA-POMCP is also
directly compared to the baseline lookahead planner by~\citet{BAPOMDP}.

First, we apply lookahead with depth $1 \& 2$ on the Tiger problem under the
same circumstance as the first experiment for increasing number of particles
($25$, $50$, $100$, $200$ \& $500$), which determines the runtime. The
resulting average episode return is plotted against the action selection time
in Figure~\ref{fig:tiger_aggregated_all}.

\begin{figure}[t] 
    \centering
    \begin{subfigure}{0.5\textwidth}
        \centering
        \includegraphics[width=1\textwidth]{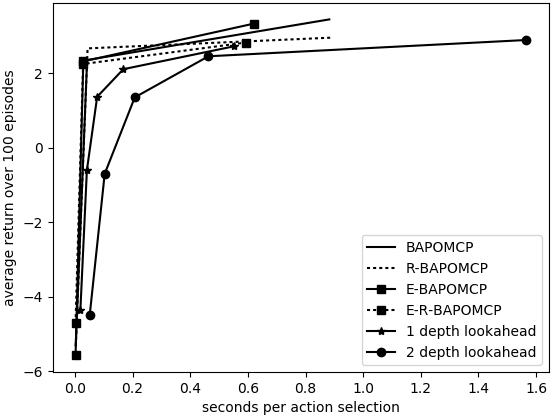}
        \caption{The average return over $100$ episodes per
        action selection time of on the Tiger problem}
        \label{fig:tiger_aggregated_all} 
    \end{subfigure}%
    ~
    \begin{subfigure}{0.5\textwidth}
        \centering
        \includegraphics[width=1\textwidth]{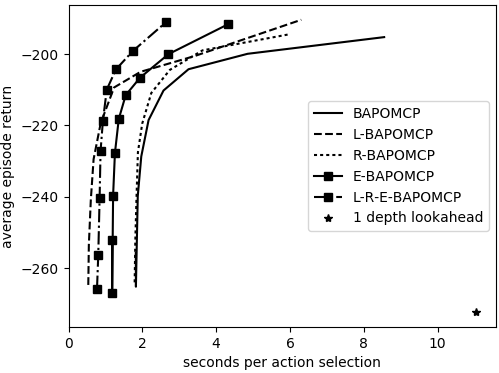}
        \caption{The average return over $100$ episodes per
        action selection time of BA-POMCP on the POSysadmin problem}
        \label{fig:6comps_aggregated} 
    \end{subfigure}
\end{figure}

The results show that most methods reach near optimal performance after $0.5$
seconds action selection time. R-BA-POMCP and E-R-BA-POMCP perform worse than
their counterparts BA-POMCP and E-BAPOMCP, which suggests that root sampling of
the dynamics actually \textit{slows down} BA-POMCP slightly. This phenomenon is
due to the fact that the Tiger problem is so small, that the overhead of
copying the augmented state and re-sampling of dynamics (during \funcName{step}
function) that root sampling avoids is negligible and does overcome the
additional complexity of root sampling.  Also note that, even though the Tiger
problem is so trivial that a lookahead of depth $1$ suffices to solve the POMDP
problem optimally, BA-POMCP still consistently outperforms this baseline.

The last experiment shows BA-POMCP and lookahead on the POSysadmin domain with
$6$ computers (which contains $55744$ counts) with a failure rate of $0.05$.
The agent was provided with an accurate belief $\chi$.\footnote 
    {
        We do not use the same prior as in the first BA-POMCP variants
        experiments since this gives uninformative results due to the fact that
        solution methods convergence to the optimal policy with respect to the
        (noisy) belief, which is different from the one with respect to the
        true model.
    } 
The results are shown in Figure~\ref{fig:6comps_aggregated}.

We were unable to get lookahead search to solve this problem: the single
instance which returned results in a reasonable amount of time (the single dot
in the lower right corner) was with a lookahead depth of $1$ (which is
insufficient for this domain) with just $50$ particles. BA-POMCP, however, was
able to perform up to $4096$ simulations within $5$ seconds and reach an
average return of approximately $-198$, utilizing a belief of $1000$ particles.
The best performing method, L-R-E-BA-POMCP requires less than $2$ seconds for
similar results, and is able to reach approximately $-190$ in less than 3
seconds.  Finally, we see that each of the individual modifications outperform
the original BA-POMCP, where Expected models seems to be the biggest
contributor.

\section{Conclusion}\label{sec:discussion}
This paper provides a scalable framework for learning in Bayes-Adaptive POMDPs.
BA-POMDPs give a principled way of balancing exploration and exploiting in RL
for POMDPs, but previous solution methods have not scaled to non-trivial
domains.  We extended the Monte Carlo Tree Search method POMCP  to BA-POMDPs
and described three modifications---Root Sampling, Linking States and Expected
Dynamics models--- to take advantage of BA-POMDP structure. We proved
convergence of the techniques and demonstrated that our methods can generate
high-quality solutions on significantly larger problems than previous methods
in the literature.

\section*{Acknowledgements}
Research supported by NSF grant \#1664923 and NWO Innovational Research Incentives Scheme Veni \#639.021.336.

\newpage
\appendix
\section{Proof of Theorem 3}

\global\long\def\cip{\overset{p}{\rightarrow}}

\global\long\def\fh{H}

\global\long\def\fhD#1{\fh_{#1}}

\global\long\def\fhDP#1#2{\fh_{#1}^{(#2)}}

\global\long\def\aoh{h}

\global\long\def\aohD#1{\aoh_{#1}}

\global\long\def\sh{STATEHISTORY}

\global\long\def\shD#1{s_{0:#1}}

\global\long\def\KD#1#2{{\mathbb{I}_{{#1},{#2}}}}

\global\long\def\DD#1#2{{\mathbb{I}_{{#1},{#2}}}}

\global\long\def\sbr{\bar{s}}

\global\long\def\sba{\check{s}}

\providecommand{\BAPOMCP}{BA-POMCP} \providecommand{\RSBAPOMCP}{RS-BA-POMCP}

While \RSBAPOMCP{} is potentially more efficient, it is not directly clear
whether it still converges to an $\epsilon$-optimal value function.
Here we show that the method is sound by showing that, when using
root sampling of the model, the distribution over \emph{full histories}
(including states, actions and observations) will converge in probability
to the same distribution when not using this additional root sampling
step.

\subsection*{Notation}

We will give an concise itemized description of the used notation.

\paragraph{Action-observation histories.}
\begin{itemize}
\item $\aohD d$ is an action-observation history at depth~$d$ of a simulation.
\item $\aohD d=\left(a_{0},z_{1},\dots,a_{d-1},z_{d}\right)$.
\end{itemize}

\paragraph{`Full' histories.}

In addition to actions and observations, full histories also include
the states.
\begin{itemize}
\item $H_{0}$ is the (unknown) full\emph{ }history (of \emph{real }experience)
at the root of the simulation: i.e., if there have been $k$ steps
of `real' experience $H_{0}=(s_{-k},a_{-k},s_{-k+1},z_{-k-1},\dots,a_{-1},s_{0},z_{0})$.
\item $\fhD d$ is a\emph{ }full\emph{ }history (of \emph{simulated }experience)
at depth $d$ in the lookahead tree: $\fhD d=\left(H_{0},a_{0},s_{1},z_{1},a_{1},s_{2},z_{2},\dots,a_{d-1},s_{d},z_{d}\right)=\left(H_{d-1},a_{d-1},s_{d},z_{d}\right)=\left\langle H_{0},\shD d,\aohD d\right\rangle $.
\item $\fhDP di$ is the full history at depth $d$ corresponding to simulation
$i$.
\item In our proof, we will also need to indicate if a particular full history
$H_{d}$ is consistent with a full history at the root of simulation:
\[
\text{Cons}(H_{0},H_{d})=\begin{cases}
1 & \text{if \ensuremath{H_{d}} is consistent with the full history at the root \ensuremath{H_{0}} , }\\
0 & \text{otherwise}.
\end{cases}
\]

\end{itemize}

\paragraph{Dynamics Function.}

We fold transition and observations function into one:
\begin{itemize}
\item $D$ denotes the dynamics model.
\item $D_{s_{t-1}a_{t-1}}^{s_{t}z_{t}}=D_{sa}^{s'z}=D_{s_{t-1},a_{t-1}}(s_{t},z_{t})=D(s_{t},z_{t}|s_{t-1},a_{t-1})=\Pr(s_{t},z_{t}|s_{t-1},a_{t-1})$.
\item $D_{sa}$ denotes the vector: $\left\langle D_{sa}^{s^{1}z^{1}},\dots,D_{sa}^{s^{\left|\mathcal{S}\right|}z^{\left|\mathcal{Z}\right|}}\right\rangle $.
\end{itemize}

\paragraph{Counts.}
\begin{itemize}
\item $\chi_{sa}^{s'z}$ denotes how often $\left\langle s',z\right\rangle $
occurred after $\left\langle s,a\right\rangle $.
\item $\chi_{sa}$ is the vector of counts for $\left\langle s,a\right\rangle $.
\item $\chi=\left\langle \chi_{s^{1}a^{1}},\dots,\chi_{s^{\left|\mathcal{S}\right|}a^{\left|\mathcal{A}\right|}}\right\rangle $
is the total collection of all such count vectors.
\item $\chi(\fhD d)$ denotes the vector of counts at simulated full history
$\fhD d$.
\item If $\chi_{0}=\chi(\fhD 0)$ is the count vector at the root of simulation,
we have that $\chi(\fhD d)=\chi_{0}+\Delta(\fhD d)$, with $\Delta(\fhD d)$
the vector of counts of all $\left(s,a,s',z\right)$ quadruples occurring
in $\fhD d$ since the root of simulation (after $\fhD 0$).
\end{itemize}

\paragraph{Dirichlet distributions.}
\begin{itemize}
\item Let $x=\left\langle x_{1}\dots x_{K}\right\rangle \in\Delta^{K}$
and $\alpha=\left\langle \alpha_{1}\dots\alpha_{K}\right\rangle $
be a count vector, then we write $Dir(x|\alpha)=\Pr(x;\alpha)=B(\mbox{\ensuremath{\alpha}})\prod_{i=1}^{K}x_{i}^{\alpha_{i}-1}$
, with $B(\alpha)=\frac{\Gamma(\sum_{i}\alpha_{i})}{\prod_{i}\Gamma(\alpha_{i})}$
the Dirichlet normalization constant, with $\Gamma$ the gamma function.
\item So, in translated in terms of dynamics function and counts, we have:

\begin{itemize}
\item for a particular $s,a$: $Dir(D_{sa}|\chi_{sa})=\Pr(D_{sa};\chi_{sa})=B(\chi_{sa})\prod_{\left\langle s',z\right\rangle \in\mathcal{S\times}\mathcal{Z}}\left(D_{sa}^{s'z}\right)^{\chi_{sa}^{s'z}-1}$.
\item we will also abuse notation and write $Dir(D|\chi)=\prod_{\left\langle s,a\right\rangle }Dir(D_{sa}|\chi_{sa})$.
\end{itemize}
\end{itemize}

\paragraph{Var.}
\begin{itemize}
\item $\dot{x}$ denotes a root sampled quantity $x$.
\item $\mathbb{I}_{\left\{ condition\right\} }$ is the indicator function
which is 1 iff condition is true and 0 otherwise.
\end{itemize}

\subsection*{Definitions}
\begin{defn} \label{def:The-full-history-BA-POMDP}
    The \emph{expected} \emph{full-history }\textbf{\emph{expected transition
    }}\emph{BA-POMDP rollout distribution} is the distribution over full
    histories of a BA-POMDP, when performing Monte-Carlo simulations according
    to a policy $\pi$. It is given by
    \begin{equation}
        P^{\pi}(\fhD{d+1})=D_{\chi(\fhD d)}(s_{d+1},z_{d+1}|a_{s},s_{d})\pi(a_{d}|\aohD d)P^{\pi}(\fhD d)\label{eq:fh-BA-POMDP_rollout-distr}
    \end{equation}
    with $P^{\pi}(\fhD 0)=b_{0}(\left\langle s_{0},\chi_{0}\right\rangle )$ the
    belief `now' (at the root of the online planning).
\end{defn}
~
\begin{defn} \label{def:The-full-history-RS-BA-POMDP}
    The \emph{empirical full-history }\textbf{\emph{root-sampling (RS)}}\emph{
        BA-POMDP rollout distribution} is the distribution over full histories
    of a BA-POMDP, when performing Monte-Carlo simulations according to a
    policy $\pi$ \textbf{in combination with root sampling} of the dynamics
    model $D$. This distribution, for a particular stage $d$, is given by \[
        \tilde{P}_{K}^{\pi}(\fhD
        d)\triangleq\frac{1}{K_{d}}\sum_{i=1}^{K_{d}}\mathbb{I}_{\left\{ \fhD
        d=\fhDP di\right\} }, \] where
\end{defn}
\begin{itemize}
\item $K$ is the number of simulations that comprise the empirical distribution.
\item $K_{d}$ is the number of simulations that reach depth $d$ (not all
simulations might be equally long).
\item $\fhDP di$ is the history specified by the $i$-th particle at stage
$d$.
\end{itemize}

\paragraph{Remark:}

\label{rem:fixed_H0}throughout this proof we assume that there is
only 1 initial count vector at the root. Or put better: we assume
that there is one unique $H_{0}$ at which all simulations start.
However, for `real' steps $t>0$ we could be in different $H_{t}^{real}$
all corresponding to the same observed real history $h_{t}^{real}$.
In this case, root sampling from the belief can be thought of root
sampling the initial full history $H_{0}\sim b(H_{t}^{real})$. As
such, our proof shows convergence in probability of
\[
\forall_{\fhD 0}\forall_{\fhD d}\quad\tilde{P}_{K_{d}}^{\pi}(\fhD d|\fhD 0)\cip P^{\pi}(\fhD d|\fhD 0).
\]
for each such sampled $H_{0}$. It is clear that that directly implies
that
\[
\forall_{\fhD d}\quad\tilde{P}_{K_{d}}^{\pi}(\fhD d)=\mathbf{E}_{H_{0}}\left[\tilde{P}_{K_{d}}^{\pi}(\fhD d|\fhD 0)\right]\cip\mathbf{E}_{H_{0}}\left[P^{\pi}(\fhD d|\fhD 0)\right]=P^{\pi}(\fhD d).
\]
In the below, we omit the explicit conditioning on $H_{0}$.

\subsection*{Proof of Main Theorem}

The proof depends on a lemma that follows below.
\begin{thm}
\label{thm:P(h)_converges-in-prob}The\emph{ }full-history RS-BA-POMDP
rollout distribution (Def.~\ref{def:The-full-history-RS-BA-POMDP})
converges in probability to full-history BA-POMDP rollout distribution~(Def.~\ref{def:The-full-history-BA-POMDP}):
\begin{equation}
\forall_{\fhD d}\quad\tilde{P}_{K_{d}}^{\pi}(\fhD d)\cip P^{\pi}(\fhD d).\label{eq:lemma_RS-rollout_converges_to_rollout}
\end{equation}
\end{thm}
\begin{proof}
For ease of notation we prove this for stage $d+1$. Note that a history
$\fhD{d+1}=\left(\fhD d,a_{d},s_{d+1},z_{d+1}\right)$, \emph{only}
differs from $\fhD d$ in that it has one extra transition for the
$\left(s_{d},a_{d},s_{d+1},z_{d+1}\right)$ quadruple, implying that
$\chi(H_{d+1})$ only differs from $\chi(\fhD d)$ in the counts $\chi_{s_{d}a_{d}}$
for $s_{d}a_{d}$. Therefore, the expression for $\tilde{P}_{K_{d}}^{\pi}(\fhD d)$
derived in Lemma~\ref{lem:fh-RS-BA-POMDP_convergence} below (cf.
equation (\ref{eq:RS-P(fhD)})) can be written in recursive form as

\begin{eqnarray*}
\tilde{P}^{\pi}(\fhD{d+1}) & = & \text{Cons}(H_{0},H_{d})\prod_{t=0}^{d}\pi(a_{t}|h_{t})\prod_{\left\langle s,a\right\rangle }\frac{B(\chi_{sa}(\fhD 0))}{B(\chi_{sa}(\fhD{d+1}))}\\
& = & \text{Cons}(H_{0},H_{d})\prod_{t=0}^{d-1}\pi(a_{t}|h_{t})\pi(a_{d}|h_{d})\prod_{\left\langle s,a\right\rangle }\frac{B(\chi_{sa}(\fhD 0))}{B(\chi_{sa}(\fhD d))}\frac{B(\chi_{sa}(\fhD d))}{B(\chi_{sa}(\fhD{d+1}))}\\
& = & \text{Cons}(H_{0},H_{d})\prod_{t=0}^{d-1}\pi(a_{t}|h_{t})\pi(a_{d}|h_{d})\left[\prod_{\left\langle s,a\right\rangle }\frac{B(\chi_{sa}(\fhD 0))}{B(\chi_{sa}(\fhD d))}\right]\left[\prod_{\left\langle s,a\right\rangle }\frac{B(\chi_{sa}(\fhD d))}{B(\chi_{sa}(\fhD{d+1}))}\right]\\
& = & \left[\text{Cons}(H_{0},H_{d})\prod_{t=0}^{d-1}\pi(a_{t}|h_{t})\prod_{\left\langle s,a\right\rangle }\frac{B(\chi_{sa}(\fhD 0))}{B(\chi_{sa}(\fhD d))}\right]\pi(a_{d}|h_{d})\frac{B(\chi_{s_{d}a_{d}}(\fhD d))}{B(\chi_{s_{d}a_{d}}(\fhD{d+1}))}\\
& = & \tilde{P}^{\pi}(\fhD d)\pi(a_{d}|\aohD d)\frac{B(\chi_{s_{d}a_{d}}(\fhD d))}{B(\chi_{s_{d}a_{d}}(\fhD{d+1}))}
\end{eqnarray*}
with base case $\tilde{P}^{\pi}(\fhD 0)=1$, and
\begin{equation}
\frac{B(\chi_{s_{d}a_{d}}(\fhD d))}{B(\chi_{s_{d}a_{d}}(\fhD{d+1}))}=\frac{B(\chi_{s_{d}a_{d}}(\fhD 0))}{B(\chi_{s_{d}a_{d}}(\fhD{d+1}))}\cdot\frac{B(\chi_{s_{d}a_{d}}(\fhD d))}{B(\chi_{s_{d}a_{d}}(\fhD 0))}=\frac{B(\chi_{s_{d}a_{d}}(\fhD 0))/B(\chi_{s_{d}a_{d}}(\fhD{d+1}))}{B(\chi_{s_{d}a_{d}}(\fhD 0))/B(\chi_{s_{d}a_{d}}(\fhD d))}\label{eq:B_ratio}
\end{equation}
the result of dividing out the contribution of the old counts for
$s_{d}a_{d}$ and multiplying in the new contribution. Now, we investigate
these terms more closely.

Again remember that the sole difference between $\fhD{d+1}=\left(\fhD d,a_{d},s_{d+1},z_{d+1}\right)$ 
and $\fhD d$ is that it has one extra transition for the $\left(s_{d},a_{d},s_{d+1},z_{d+1}\right)$
quadruple. Let us write $T=\sum_{\left(s',z\right)}\chi_{s_{d}a_{d}}^{s'z}(\fhD d)$
for the total of the counts for $s_{d},a_{d}$ and $N=\chi_{s_{d}a_{d}}^{s_{d+1}z_{d+1}}(\fhD d)$
for the number of counts for that such a transition was to $\left(s_{d+1}z_{d+1}\right)$.
Because $\fhD{d+1}$ only has 1 extra transition, we also know that
for this history, the total counts is one higher: $\sum_{\left(s',z\right)}\chi_{s_{d}a_{d}}^{s'z}(\fhD{d+1})=T+1$
and since that transition was to $(s_{d+1}z_{d+1})$ the counts $\chi_{s_{d}a_{d}}^{s_{d+1}z_{d+1}}(\fhD{d+1})=N+1$.
Now let us expand the term from (\ref{eq:B_ratio}):
\begin{eqnarray*}
\frac{B(\chi_{s_{d}a_{d}}(\fhD d))}{B(\chi_{s_{d}a_{d}}(\fhD{d+1}))} & = & \frac{\Gamma(T)/\prod_{s'z}\Gamma(\chi_{s_{d}a_{d}}^{s'z}(\fhD d))}{\Gamma(T+1)/\prod_{s'z}\Gamma(\chi_{s_{d}a_{d}}^{s'z}(\fhD{d+1}))}\\
& = & \frac{\Gamma(T)}{\Gamma(T+1)}\frac{\prod_{s'z}\Gamma(\chi_{s_{d}a_{d}}^{s'z}(\fhD{d+1}))}{\prod_{s'z}\Gamma(\chi_{s_{d}a_{d}}^{s'}(\fhD d))}\\
& = & \frac{\Gamma(T)}{\Gamma(T+1)}\frac{\Gamma(\chi_{s_{d}a_{d}}^{s_{d+1}z_{d+1}}(\fhD{d+1}))\prod_{s'z\neq\left(s_{d+1}z_{d+1}\right)}\Gamma(\chi_{s_{d}a_{d}}^{s'z}(\fhD{d+1}))}{\Gamma(\chi_{s_{d}a_{d}}^{s_{d+1}z_{d+1}}(\fhD d))\prod_{s'z\neq\left(s_{d+1}z_{d+1}\right)}\Gamma(\chi_{s_{d}a_{d}}^{s'z}(\fhD d))}\\
& = & \frac{\Gamma(T)}{\Gamma(T+1)}\frac{\Gamma(\chi_{s_{d}a_{d}}^{s_{d+1}z_{d+1}}(\fhD{d+1}))}{\Gamma(\chi_{s_{d}a_{d}}^{s_{d+1}z_{d+1}}(\fhD d))}=\frac{\Gamma(T)}{\Gamma(T+1)}\frac{\Gamma(N+1)}{\Gamma(N)}
\end{eqnarray*}
Now, the gamma function has the property that $\Gamma(x+1)=x\Gamma(x)$
\citep{DeGroot04}, which means that we get
\[
=\frac{\Gamma(T)}{T\Gamma(T)}\frac{N\Gamma(N)}{\Gamma(N)}=\frac{N}{T}.
\]
Therefore we get
\[
\frac{B(\chi_{s_{d}a_{d}}(\fhD d))}{B(\chi_{s_{d}a_{d}}(\fhD{d+1}))}=\frac{\chi_{s_{d}a_{d}}^{s_{d+1}z_{d+1}}(\fhD d)}{\sum_{\left(s',z\right)}\chi_{s_{d}a_{d}}^{s'z}(\fhD d)}
\]
and thus
\begin{equation}
\tilde{P}^{\pi}(\fhD{d+1})=\tilde{P}^{\pi}(\fhD d)\pi(a_{d}|\aohD d)\frac{\chi_{s_{d}a_{d}}^{s_{d+1}z_{d+1}}(\fhD d)}{\sum_{\left(s',z\right)}\chi_{s_{d}a_{d}}^{s'z}(\fhD d)}.\label{eq:recursive-RS-P(H)__induction-step}
\end{equation}
the r.h.s. of this equation is identical to (\ref{eq:fh-BA-POMDP_rollout-distr})
except for the difference in between $\tilde{P}^{\pi}(\fhD d)$ and
$P^{\pi}(\fhD d)$. This can be resolved by forward induction with
base step: $\tilde{P}^{\pi}(\fhD 0)=b_{0}(\left\langle s_{0},\chi_{0},\psi_{0}\right\rangle )=P^{\pi}(\fhD 0)$,
and the induction step (show $\tilde{P}^{\pi}(\fhD{d+1})=P^{\pi}(\fhD{d+1})$
given $\tilde{P}^{\pi}(\fhD d)=P^{\pi}(\fhD d)$) directly following
from (\ref{eq:fh-BA-POMDP_rollout-distr}) and (\ref{eq:recursive-RS-P(H)__induction-step}).
Therefore we can conclude that $\forall_{d}\quad\tilde{P}^{\pi}(\fhD d)=P^{\pi}(\fhD d).$

Since Lemma~\ref{lem:fh-RS-BA-POMDP_convergence} establishes that
$\forall_{\fhD d}\quad\tilde{P}_{K_{d}}^{\pi}(\fhD d)\cip\tilde{P}^{\pi}(\fhD d)$,
we directly have
\[
\forall_{\fhD d}\quad\tilde{P}_{K_{d}}^{\pi}(\fhD d)\cip P^{\pi}(\fhD d),
\]
thus proving the result.
\end{proof}
The proof depends on the following lemma:
\begin{lem}
\label{lem:fh-RS-BA-POMDP_convergence} The full-history RS-BA-POMDP
rollout distribution converges in probability to the following quantity:
\begin{equation}
\forall_{\fhD d}\quad\tilde{P}_{K_{d}}^{\pi}(\fhD d)\cip b_{0}(s_{0})\left[\prod_{t=1}^{d}\pi(a_{t-1}|h_{t-0})\right]\left[\prod_{\left\langle s,a\right\rangle }\frac{B(\chi_{sa}(\fhD 0))}{B(\chi_{sa}(\fhD d)}\right]\label{eq:fh-RS-BA-POMDP-rollout-distribution_convergence}
\end{equation}
with $B(\alpha)=\frac{\Gamma(\alpha_{1}+\ldots\cdot+\alpha_{k})}{\Gamma(\alpha_{1})\cdot\ldots\cdot\Gamma(\alpha_{k})}$
the normalization term of a Dirichlet distribution with parametric
vector $\alpha$.\end{lem}
\begin{proof}
Via the weak law of large numbers, we have that the empirical mean
of a random variable converges in probability to its expectation.
\[
\forall_{\fhD d}\quad\tilde{P}_{K_{d}}^{\pi}(\fhD d)\cip\frac{1}{K_{d}}\sum_{i=1}^{K_{d}}\mathbb{I}_{\left\{ \fhD d=\fhDP di\right\} }\cip\mathbf{E}\left[\mathbb{I}_{\left\{ \fhD d=\fhDP di\right\} }\right]
\]
This expectation can be rewritten as follows
\begin{equation}
\mathbf{E}\left[\mathbb{I}_{\left\{ \fhD d=\fhDP di\right\} }\right]=\sum_{\fhDP di}\tilde{P}^{\pi}\left(\fhDP di\right)\mathbb{I}_{\left\{ \fhD d=\fhDP di\right\} }=\tilde{P}^{\pi}\left(\fhD d\right)
\end{equation}
where $\tilde{P}^{\pi}(\fhD d)$ denotes the (true, non-empirical)
probability that the RS-BA-POMDP rollout generates full history $\fhD d$.
This is an expectation over the root sampled model $\dot{D}$:
\begin{eqnarray}
\tilde{P}^{\pi}(\fhD d) & = & \int\tilde{P}^{\pi}\left(\fhD d|\dot{D}\right)Dir(\dot{D}|\dot{\chi})d\dot{D}\\
& = & \int\left[\text{Cons}(H_{0},H_{d})\prod_{t=1}^{d}\dot{D}(s_{t},z_{t}|s_{t-1},a_{t-1})\pi(a_{t-1}|h_{t-1})\right]Dir(\dot{D}|\dot{\chi})d\dot{D}\\
& = & \text{Cons}(H_{0},H_{d})\left[\prod_{t=1}^{d}\pi(a_{t-1}|h_{t-1})\right]\left(\int\left[\prod_{t=1}^{d}\dot{D}(s_{t},z_{t}|s_{t-1},a_{t-1})\right]Dir(\dot{D}|\dot{\chi})d\dot{D}\right)
\end{eqnarray}
Where $\text{Cons}(H_{0},H_{d})$ is a term that indicates whether
(takes value 1 if) $H_{d}$ is consistent with the full history at
the root $H_{0}$.\footnote{An earlier version of this proof~(\cite{oliehoek2014best})
contained a term $b_{0}(s_{0})$ instead of $\text{Cons}(H_{0},H_{d})$,
which fails to recognize that this proof assumes $H_{0}$ to be fixed.
See also the remark \vpageref{rem:fixed_H0}.}

Now we can exploit the fact that only the Dirichlet for the transitions
specified by $H_{d}$ matter.
\begin{align}
& \int\left[\prod_{t=1}^{d}\dot{D}(s_{t},z_{t}|s_{t-1},a_{t-1})\right]Dir(\dot{D}|\chi_{0})d\dot{D}\\
= & \text{\{split up the integral over one big vector into integrals over smaller vectors\}}\nonumber \\
& \int\dots\int\left[\prod_{t=1}^{d}\dot{D}_{s_{t-1},a_{t-1}}^{s_{t},z_{t}}\right]\left[\prod_{\left\langle s,a\right\rangle }Dir(\dot{D}_{sa}|\chi_{sa}(\fhD 0))\right]d\dot{D}_{s^{1}a^{1}}\dots d\dot{D}_{s^{\left|\mathcal{S}\right|}a^{\left|\mathcal{A}\right|}}\\
= & \text{\{reorder the transition probabilities: \ensuremath{\Delta_{\chi}^{sas'z}(\fhD d)}is the number of occurences of \ensuremath{\left(s,a,s',z\right)}in \ensuremath{\fhD d}\}}\nonumber \\
& \int\dots\int\left[\prod_{\left\langle s,a\right\rangle }\prod_{\left\langle s',z\right\rangle }\left(\dot{D}_{sa}^{s'z}\right)^{\Delta_{\chi}^{sas'z}(\fhD d)}\right]\left[\prod_{\left\langle s,a\right\rangle }Dir(\dot{D}_{sa}|\chi_{sa}(\fhD 0))\right]d\dot{D}_{s^{1}a^{1}}\dots d\dot{D}_{s^{\left|\mathcal{S}\right|}a^{\left|\mathcal{A}\right|}}\\
= & \int\dots\int\left[\prod_{\left\langle s,a\right\rangle }\prod_{\left\langle s',z\right\rangle }\left(\dot{D}_{sa}^{s'z}\right)^{\Delta_{\chi}^{sas'z}(\fhD d)}\right]\left[\prod_{\left\langle s,a\right\rangle }B(\dot{\chi}_{sa})\prod_{\left\langle s',z\right\rangle }\left(\dot{D}_{sa}^{s'z}\right)^{\chi_{0}^{sas'z}-1}\right]d\dot{D}_{s^{1}a^{1}}\dots d\dot{D}_{s^{\left|\mathcal{S}\right|}a^{\left|\mathcal{A}\right|}}\\
= & \int\dots\int\left[\prod_{\left\langle s,a\right\rangle }\left(\left[\prod_{\left\langle s',z\right\rangle }\left(\dot{D}_{sa}^{s'z}\right)^{\Delta_{\chi}^{sas'z}(\fhD d)}\right]\left[B(\dot{\chi}_{sa})\prod_{\left\langle s',z\right\rangle }\left(\dot{D}_{sa}^{s'z}\right)^{\chi_{0}^{sas'z}-1}\right]\right)\right]d\dot{D}_{s^{1}a^{1}}\dots d\dot{D}_{s^{\left|\mathcal{S}\right|}a^{\left|\mathcal{A}\right|}}\\
= & \int\dots\int\left[\prod_{\left\langle s,a\right\rangle }\left(B(\dot{\chi}_{sa})\left[\prod_{\left\langle s',z\right\rangle }\left(\dot{D}_{sa}^{s'z}\right)^{\Delta_{\chi}^{sas'z}(\fhD d)}\right]\left[\prod_{\left\langle s',z\right\rangle }\left(\dot{D}_{sa}^{s'z}\right)^{\chi_{0}^{sas'z}-1}\right]\right)\right]d\dot{D}_{s^{1}a^{1}}\dots d\dot{D}_{s^{\left|\mathcal{S}\right|}a^{\left|\mathcal{A}\right|}}\\
= & \int\dots\int\left[\prod_{\left\langle s,a\right\rangle }B(\dot{\chi}_{sa})\prod_{\left\langle s',z\right\rangle }\left(\dot{D}_{sa}^{s'z}\right)^{\chi_{0}^{sas'z}-1+\Delta_{\chi}^{sas'z}(\fhD d)}\right]d\dot{D}_{s^{1}a^{1}}\dots d\dot{D}_{s^{\left|\mathcal{S}\right|}a^{\left|\mathcal{A}\right|}}\label{eq:before-reorder}
\end{align}
Now we reverse the order of integration and multiplication, which
is possible since the different $s,a$ pairs over which we integrate
are disjoint.\footnote{E.g, consider two sets $A_{1}=\left\{ a_{1}^{(1)},a_{1}^{(2)}\right\} $
and $A_{2}=\left\{ a_{2}^{(1)},a_{2}^{(2)},a_{2}^{(3)}\right\} $.
Equation (\ref{eq:before-reorder}) is of the same form as
\begin{align*}
\sum_{a_{1}\in A_{1}}\sum_{a_{2}\in A_{2}}\prod_{i=1}^{2}a_{i} & =\sum_{a_{1}\in A_{1}}\sum_{a_{2}\in A_{2}}a_{1}a_{2}=a_{1}^{(1)}a_{2}^{(1)}+a_{1}^{(1)}a_{2}^{(2)}+a_{1}^{(1)}a_{2}^{(3)}+a_{1}^{(2)}a_{2}^{(1)}+a_{1}^{(2)}a_{2}^{(2)}+a_{1}^{(2)}a_{2}^{(3)}\\
& =a_{1}^{(1)}\left(a_{2}^{(1)}+a_{2}^{(2)}+a_{2}^{(3)}\right)+a_{1}^{(2)}\left(a_{2}^{(1)}+a_{2}^{(2)}+a_{2}^{(3)}\right)=\left(a_{1}^{(1)}+a_{1}^{(2)}\right)\left(a_{2}^{(1)}+a_{2}^{(2)}+a_{2}^{(3)}\right)\\
& =\left[\sum_{a_{1}\in A_{1}}a_{1}\right]\left[\sum_{a{}_{2}\in A_{2}}a_{2}\right]=\prod_{i=1}^{2}\sum_{a_{i}\in A_{i}}a_{i}
\end{align*}
} We obtain:
\begin{align}
= & \prod_{\left\langle s,a\right\rangle }B(\chi_{sa}(\fhD 0))\int\prod_{\left\langle s',z\right\rangle }\left(\dot{D}_{sa}(s',z)\right)^{\chi_{0}^{sas'z}+\Delta_{\chi}^{sas'z}(\fhD d)-1}d\dot{D}_{sa}\\
= & \text{\{since we integrate over the entire vector \ensuremath{\dot{D}_{sa}}, the integral equals \ensuremath{1/B(\chi_{sa}(\fhD 0)+\Delta_{\chi}^{sa}(\fhD d))}\}}\nonumber \\
& \prod_{\left\langle s,a\right\rangle }B(\chi_{sa}(\fhD 0))\frac{1}{B(\chi_{sa}(\fhD 0)+\Delta_{\chi}^{sa}(\fhD d))}\\
= & \prod_{\left\langle s,a\right\rangle }\frac{B(\chi_{sa}(\fhD 0))}{B(\chi_{sa}(\fhD d))}
\end{align}
Therefore
\begin{equation}
\tilde{P}^{\pi}(\fhD d)=\text{Cons}(H_{0},H_{d})\left[\prod_{t=0}^{d-1}\pi(a_{t}|h_{t})\right]\left[\prod_{\left\langle s,a\right\rangle }\frac{B(\chi_{sa}(\fhD 0))}{B(\chi_{sa}(\fhD d))}\right],\label{eq:RS-P(fhD)}
\end{equation}
proving (\ref{eq:fh-RS-BA-POMDP-rollout-distribution_convergence}).
\end{proof}

\bibliographystyle{abbrvnat}
\setcitestyle{authoryear,open={((},close={))}}
\bibliography{ref}

\end{document}